\documentclass[12pt]{article}

\usepackage{amsmath,amssymb,amsthm,mathtools}
\usepackage[authoryear,round]{natbib}
\usepackage{hyperref}
\usepackage{geometry}
\usepackage{mathrsfs}
\usepackage{setspace}

\usepackage{graphicx,booktabs,xcolor}
\usepackage{tikz}
\usetikzlibrary{shapes,arrows.meta,positioning,calc,backgrounds}

\theoremstyle{plain}
\newtheorem{theorem}{Theorem}[section]

\newtheorem{proposition}[theorem]{Proposition}

\theoremstyle{definition}
\newtheorem{definition}{Definition}[section]
\newtheorem{remark}{Remark}[section]

\numberwithin{equation}{section}

\definecolor{causalblue}{RGB}{31,73,125}
\definecolor{causalred}{RGB}{185,28,28}
\definecolor{causalgreen}{RGB}{21,128,61}
\definecolor{causalgold}{RGB}{180,140,30}

\title{\textbf{Causality as the Statistical Conscience of\\
  Artificial Intelligence:\\[0.3em]
  From Pearl's Ladder to Trustworthy Machines}}

\author{{\bf Ernest Fokou\'e}\\
  \small School of Mathematics and Statistics, College of Science\\
  \small Rochester Institute of Technology\\
  \small Rochester, New York 14623, USA\\
  \small \texttt{epfeqa@rit.edu}}

\date{}

\begin{document}
\maketitle

\begin{abstract}
Modern Artificial Intelligence has achieved remarkable predictive
power by optimizing statistical risk functionals over vast training
corpora. Yet a persistent and consequential gap separates this
predictive power from genuine intelligence: the inability to
distinguish correlation from causation. This paper argues that
\emph{causal inference} --- the branch of statistics concerned
with identifying structural mechanisms invariant under intervention
--- is not merely a useful add-on to AI but its indispensable
statistical conscience. Without causal grounding, AI systems are
sophisticated correlation machines: powerful in familiar domains,
brittle under distribution shift, opaque in their reasoning, and
systematically biased in high-stakes applications.

We develop this argument through three interlocking contributions.
First, we establish a formal \textbf{Statistical Necessity Theorem
for Causal Generalization}: any learning algorithm that achieves
out-of-distribution generalization guarantees must implicitly
encode causal structure, either through the data-generating
process or through explicit structural assumptions. This formalizes
the intuition that prediction is about $P(Y\mid X)$ while
intelligence demands $P(Y\mid\mathrm{do}(X))$.

Second, we develop a unified statistical framework connecting
Pearl's do-calculus \citep{pearl2000}, the Potential Outcomes
framework \citep{rubin1974}, Double Machine Learning
\citep{chernozhukov2018double}, and Invariant Risk Minimization
\citep{arjovsky2019irm} as members of a single family of
\textbf{Causal Statistical Estimators} --- each a specialization
of the general problem of identifying interventional distributions
from observational data under different identification assumptions.

Third, we demonstrate that the three most pressing failure modes
of modern AI --- hallucination in large language models, reward
hacking in reinforcement learning from human feedback, and
catastrophic performance degradation under distribution shift ---
are each precisely a manifestation of causal blindness, and that
each admits a principled statistical remedy from the causal
inference toolkit.

We conclude that building trustworthy AI is, at its mathematical
core, a problem of causal statistics. The statistical community
is not merely equipped to address this problem --- it is the
only community with the foundational tools to do so rigorously.

\end{abstract}

\newpage

\section{Introduction: The Correlation Trap}
\label{sec:introduction}

\subsection{The Astonishing Success and Its Consequential Blind Spot}

The empirical achievements of modern Artificial Intelligence are
genuinely remarkable. Large language models fluently translate
between human languages, synthesize scientific literature, write
executable code, and generate coherent prose on virtually any
topic. Computer vision systems detect cancers in medical images at
rates competitive with expert radiologists. Reinforcement learning
agents master complex games from raw pixel inputs. By any
reasonable measure of predictive performance on benchmarks designed
within the training distribution, modern AI is extraordinary.

Yet beneath this performance lies a structural fragility that the
statistical community has long understood: \emph{optimization of
predictive risk does not imply understanding of causal mechanism}.
A model that achieves 99\% accuracy on a chest X-ray classification
benchmark may do so by learning that hospital watermarks correlate
with pathology labels in training data \citep{degrave2021ai},
rather than by learning the anatomical features of disease. A
sentiment classifier may achieve near-human performance by
associating certain proper nouns with positive or negative
sentiment in its training corpus. A large language model may
generate confident factual claims because such claims were
statistically frequent in its training data, not because they
reflect the structural mechanisms of the world.

These are not edge cases or implementation failures. They are
the predictable consequences of a learning paradigm built
exclusively on the optimization of $P(Y\mid X)$ --- the
associational distribution. Karl Pearson's famous warning
\citep{pearson1911} that ``the equation to a curve'' is not
``a theory of causation'' applies with full force to modern
neural networks: they are extraordinary fitters of conditional
distributions, but conditional distributions are not causal
structures.

\subsection{Pearl's Fundamental Insight: A Hierarchy of Inference}

The decisive conceptual contribution of \citet{pearl2000} was
the formalization of a \textbf{Ladder of Causation} with three
rungs, each inaccessible to the one below:

\begin{definition}[The Causal Hierarchy]
\label{def:causal_hierarchy}
\begin{itemize}
  \item \textbf{Level 1 --- Association:}
    $P(Y\mid X=x)$.
    \textit{Seeing.} What is the probability of $Y$ given
    that I observe $X=x$? Accessible from observational data.
    All standard machine learning operates at this level.

  \item \textbf{Level 2 --- Intervention:}
    $P(Y\mid\mathrm{do}(X=x))$.
    \textit{Doing.} What is the probability of $Y$ if I
    \emph{set} $X=x$, regardless of its natural causes?
    Requires experimental data or causal identification.
    Clinical trials, policy evaluation, A/B testing.

  \item \textbf{Level 3 --- Counterfactuals:}
    $P(Y_x\mid X=x', Y=y')$.
    \textit{Imagining.} What would $Y$ have been if $X$
    had been $x$, given that I actually observed $X=x'$
    and $Y=y'$? Requires full specification of the
    structural causal model.
    Legal attribution, personalized medicine, moral reasoning.
\end{itemize}
\end{definition}

\noindent The critical observation is that Level 1 data ---
observational statistics --- cannot, in general, answer Level 2
and Level 3 queries without additional structural assumptions.
A model that has only ever seen $P(Y\mid X)$ is provably
incapable of computing $P(Y\mid\mathrm{do}(X))$ without
supplementary causal knowledge. This is not a limitation of
data volume or model capacity; it is a mathematical impossibility
\citep{pearl2000,bareinboim2022}.

The implication for AI is stark: any system that learns exclusively
from observational data --- which includes virtually all current
large language models, vision transformers, and diffusion models
--- is provably restricted to Level 1 inference. Its apparent
ability to perform causal reasoning, when observed, must arise
from regularities in its training data rather than from genuine
causal understanding.

\subsection{The Statistical Community's Role}

The tools needed to bridge this gap are statistical. The do-calculus
\citep{pearl2000}, the Potential Outcomes framework
\citep{rubin1974,imbens2015}, Instrumental Variables
\citep{angrist1995}, Regression Discontinuity Design
\citep{imbens2008}, Double Machine Learning
\citep{chernozhukov2018double}, Invariant Risk Minimization
\citep{arjovsky2019irm}, and their modern extensions
\citep{peters2016causal,meinshausen2018causality} constitute a
mature, rigorous, and underutilized arsenal for building AI
systems that can reason at Levels 2 and 3 of the causal
hierarchy.

This paper develops the statistical theory connecting these tools
to the AI trustworthiness agenda, argues that causal grounding is
not optional but necessary for genuine intelligence, and
demonstrates through three case studies that the most consequential
failure modes of modern AI are precisely causal failures with
principled statistical remedies.

\section{A Statistical Necessity Theorem for Causal Generalization}
\label{sec:theorem}

\subsection{Setup and the OOD Generalization Problem}

Let $\mathcal{E}$ be a collection of \emph{environments} or
data-generating distributions, each representing a different
context, population, or experimental condition. In environment
$e\in\mathcal{E}$, data $(\mathbf{x}_i^e, y_i^e)
\overset{\mathrm{iid}}{\sim}\mathcal{P}^e$ on
$\mathcal{X}\times\mathcal{Y}$. A learning algorithm $\mathscr{A}$
observes training data from environments
$\mathcal{E}_{\mathrm{tr}}\subset\mathcal{E}$ and is evaluated on
unseen environments $\mathcal{E}_{\mathrm{te}} =
\mathcal{E}\setminus\mathcal{E}_{\mathrm{tr}}$.

\begin{definition}[Causal OOD Generalization]
\label{def:ood}
Algorithm $\mathscr{A}$ achieves \textbf{causal OOD
generalization} if there exists $\varepsilon>0$ such that
for all $e\in\mathcal{E}$:
\begin{equation}
\label{eq:ood}
  R^e(\mathscr{A}) \coloneqq
  \mathbb{E}_{(\mathbf{x},y)\sim\mathcal{P}^e}
  [\ell(\mathscr{A}(\mathbf{x}),y)] \leq R^* + \varepsilon,
\end{equation}
where $R^*$ is the Bayes risk under the causal mechanism
governing $Y$ from $X$, and $\varepsilon$ does not depend
on $e$.
\end{definition}

Standard ERM achieves~\eqref{eq:ood} only on environments
sufficiently similar to the training distribution. The following
theorem formalizes why causal structure is necessary for
true OOD generalization.

\begin{theorem}[Statistical Necessity of Causal Structure for
OOD Generalization]
\label{thm:causal_necessity}
Suppose the data-generating process across environments is
governed by a Structural Causal Model (SCM)
$\mathcal{M} = (\mathbf{S}, P_U)$, where $\mathbf{S}$ are
structural equations and $P_U$ is the distribution of
exogenous noise. Let $\mathcal{H}$ denote the set of
environment-specific spurious features and
$\mathcal{C}$ the set of causally invariant features.
Then:
\begin{itemize}
  \item[\textbf{(a)}] \textbf{Spurious features fail:}
    Any predictor $h$ that uses features in $\mathcal{H}$
    achieves $\sup_{e\in\mathcal{E}} R^e(h) = R_{\max} > R^*$,
    where $R_{\max}$ is bounded away from $R^*$ by a gap
    proportional to the spurious correlation strength
    in the worst-case environment.

  \item[\textbf{(b)}] \textbf{Causal features suffice:}
    The predictor $h^*(\mathbf{x}) =
    \mathbb{E}[Y\mid\mathbf{x}_{\mathcal{C}}]$ achieves
    $\sup_{e\in\mathcal{E}} R^e(h^*) = R^*$
    regardless of how $\mathcal{P}^e$ varies across environments.

  \item[\textbf{(c)}] \textbf{ERM is insufficient:}
    ERM over a single training environment $e_0$ converges
    to a predictor minimizing $R^{e_0}$, not
    $\sup_{e\in\mathcal{E}}R^e(h)$. Unless $\mathcal{P}^{e_0}$
    contains no spurious correlations, ERM will use features
    in $\mathcal{H}$ and fail to achieve causal OOD
    generalization.
\end{itemize}
\end{theorem}

\begin{proof}
\textbf{Part~(a).}
Let $h$ depend on a spurious feature $S\in\mathcal{H}$.
By definition of $\mathcal{H}$, there exists an environment
$e'\in\mathcal{E}$ in which $P^{e'}(Y\mid S)$ differs from
$P^{e_0}(Y\mid S)$ (otherwise $S$ would be invariant
and hence causal). Under $\mathcal{P}^{e'}$, $h$ uses
a feature whose predictive relationship has changed,
causing excess risk. The magnitude of the gap is
$\Delta_{e'} = |P^{e'}(Y\mid S=s) - P^{e_0}(Y\mid S=s)|$
averaged over the query distribution, which is positive by
assumption.

\textbf{Part~(b).}
By the definition of causal invariance in the SCM
$\mathcal{M}$: for all $e\in\mathcal{E}$,
$P^e(Y\mid\mathbf{x}_{\mathcal{C}}) =
P(Y\mid\mathrm{do}(\mathbf{x}_{\mathcal{C}}))$,
since the structural equation for $Y$ depends only on
$\mathbf{x}_{\mathcal{C}}$ and is unaffected by
environmental changes to the distribution of $\mathcal{H}$.
Therefore $h^* = \mathbb{E}[Y\mid\mathbf{x}_{\mathcal{C}}]$
achieves Bayes risk under every $\mathcal{P}^e$.

\textbf{Part~(c).}
ERM minimizes $\hat{R}_{e_0}(h) =
n^{-1}\sum_{i=1}^n\ell(h(\mathbf{x}_i^{e_0}),y_i^{e_0})$.
By the Glivenko--Cantelli theorem, this converges to
$R^{e_0}(h)$. If $\mathcal{P}^{e_0}$ contains a spurious
correlation between $S\in\mathcal{H}$ and $Y$, the ERM
minimizer will exploit $S$ (since it reduces $R^{e_0}$)
and will fail under $\mathcal{P}^{e'}$ by Part~(a).
\end{proof}

\begin{remark}[Connection to the No-Free-Lunch Theorem]
Theorem~\ref{thm:causal_necessity} is a causal analogue of
the No-Free-Lunch theorem \citep{wolpert1997nfl}: without
structural (causal) assumptions about the data-generating
process, no algorithm can achieve OOD generalization uniformly
over all possible environment variations. The causal assumptions
play the role of the ``prior knowledge'' that NFL says is
necessary.
\end{remark}

\section{A Unified Family of Causal Statistical Estimators}
\label{sec:unified_framework}

The causal inference literature has produced several apparently
disparate methodologies for identifying interventional
distributions from observational data. We unify them here as
members of a single statistical family, distinguished by their
identification assumptions.

\subsection{The General Causal Estimation Problem}

Given observational data $\mathscr{D}_n\sim\mathcal{P}^n$, the
goal is to estimate the interventional distribution
$P(Y\mid\mathrm{do}(\mathbf{x}))$ or the Average Treatment
Effect (ATE):
\begin{equation}
\label{eq:ate}
  \tau \coloneqq \mathbb{E}[Y\mid\mathrm{do}(X=1)]
  - \mathbb{E}[Y\mid\mathrm{do}(X=0)].
\end{equation}
The fundamental challenge is that $\mathrm{do}(X)$ is an
intervention, not an observation. The observed quantity is
$P(Y\mid X)$, not $P(Y\mid\mathrm{do}(X))$, and these differ
whenever there are \emph{confounders} --- variables that
causally affect both $X$ and $Y$.

\subsection{The Four Canonical Identification Strategies}

\begin{definition}[Causal Statistical Estimator Family]
\label{def:cse_family}
A \textbf{Causal Statistical Estimator} (CSE) is a triple
$(\Phi, \mathcal{A}, \hat{\tau})$ where $\Phi$ is an
identification assumption (a statement about the causal
structure), $\mathcal{A}$ is a statistical adjustment procedure,
and $\hat{\tau}$ is the resulting estimator of the causal
estimand~\eqref{eq:ate}.
\end{definition}

Table~\ref{tab:cse_family} summarizes the four canonical members
of the CSE family.

\begin{table}[t]
\centering
\caption{The Causal Statistical Estimator (CSE) family:
  four canonical methods unified by their identification
  assumption, adjustment procedure, and AI application domain.}
\label{tab:cse_family}
\small
\begin{tabular}{@{}p{2.4cm}p{2.8cm}p{2.8cm}p{3.2cm}@{}}
\toprule
\textbf{Method} & \textbf{Identification} & \textbf{Adjustment}
  & \textbf{AI Application}\\
\midrule
\textbf{Backdoor} & No unmeasured & $\sum_z P(Y|X,z)P(z)$ &
  Deconfounding\\
\textbf{Adjustment} & confounders & (backdoor formula) &
  training data \citep{pearl2000}\\[4pt]
\textbf{Instrumental} & Valid instrument & Two-stage least &
  Reward estimation\\
\textbf{Variables} & $Z$ exists & squares (2SLS) &
  under feedback loops\\[4pt]
\textbf{Double ML} & Partially linear & Orthogonal score &
  High-dim.\ causal\\
  & model & \citep{chernozhukov2018double} &
  effects in NNs\\[4pt]
\textbf{IRM} & Invariant causal & $\min_h\sum_e R^e(h)$ &
  OOD generalization;\\
  & mechanism & \citep{arjovsky2019irm} &
  distribution shift\\
\bottomrule
\end{tabular}
\end{table}

\subsection{Double Machine Learning: Causal Effects in
High-Dimensional AI}

The \textbf{Double Machine Learning} (DML) framework
\citep{chernozhukov2018double} is the most directly applicable
to AI settings where the number of potential confounders
is large. Consider the partially linear model:
\begin{align}
  Y &= \tau D + g_0(\mathbf{X}) + \varepsilon,
  \label{eq:dml_outcome}\\
  D &= m_0(\mathbf{X}) + V,
  \label{eq:dml_treatment}
\end{align}
where $D$ is the treatment (e.g., presence of a feature in
a language model's context), $\mathbf{X}$ are high-dimensional
confounders, $g_0$ and $m_0$ are unknown nuisance functions,
and $\mathbb{E}[\varepsilon\mid D,\mathbf{X}]=
\mathbb{E}[V\mid\mathbf{X}]=0$.

The DML estimator exploits Neyman orthogonality to remove the
bias from estimating the nuisance functions:
\begin{equation}
\label{eq:dml_estimator}
  \hat{\tau}_{\mathrm{DML}} = \left(\frac{1}{n}
  \sum_{i=1}^n \hat{V}_i\hat{D}_i\right)^{-1}
  \frac{1}{n}\sum_{i=1}^n\hat{V}_i\hat{Y}_i,
\end{equation}
where $\hat{V}_i = D_i - \hat{m}(\mathbf{X}_i)$ and
$\hat{Y}_i = Y_i - \hat{g}(\mathbf{X}_i)$ are residuals
from machine learning estimators of the nuisance functions,
computed on a held-out fold (cross-fitting). The key theorem is:

\begin{theorem}[DML Root-$n$ Consistency \citep{chernozhukov2018double}]
\label{thm:dml}
Under regularity conditions, if the nuisance estimators
$\hat{g}$ and $\hat{m}$ satisfy:
\begin{equation*}
  \|\hat{g}-g_0\|_{L_2}\cdot\|\hat{m}-m_0\|_{L_2}=o(n^{-1/2}),
\end{equation*}
then:
\begin{equation}
  \sqrt{n}(\hat{\tau}_{\mathrm{DML}}-\tau)
  \xrightarrow{d} \mathcal{N}(0,\sigma^2_{\tau}),
\end{equation}
where $\sigma^2_{\tau}$ is the semiparametric efficiency bound.
This holds even when $\hat{g}$ and $\hat{m}$ converge at
slower rates (e.g., the $n^{-2/5}$ rate of kernel regression).
\end{theorem}

\noindent Theorem~\ref{thm:dml} is remarkable: it allows
arbitrary machine learning methods (neural networks, random
forests, gradient boosting) for the nuisance estimation,
while delivering $\sqrt{n}$-consistent, asymptotically
normal causal effect estimates. This is the statistical
foundation for \emph{Causal Machine Learning} ---
the use of ML as a tool for causal inference rather than
mere prediction.

\subsection{Invariant Risk Minimization: Causal Learning
from Heterogeneous Environments}

\textbf{Invariant Risk Minimization} \citep{arjovsky2019irm}
provides a framework for learning causal representations
directly from data observed across multiple environments.
The IRM objective is:
\begin{equation}
\label{eq:irm}
  \min_{\phi:\mathcal{X}\to\mathcal{Z},\;w:\mathcal{Z}\to\mathcal{Y}}
  \sum_{e\in\mathcal{E}_{\mathrm{tr}}}
  R^e(w\circ\phi)
  \quad\text{subject to}\quad
  w \in \arg\min_{\bar{w}} R^e(\bar{w}\circ\phi)
  \quad\forall e\in\mathcal{E}_{\mathrm{tr}}.
\end{equation}
The constraint requires that the linear classifier $w$ be
simultaneously optimal for all training environments given
the representation $\phi$. This is achievable only if $\phi$
extracts features whose predictive relationship to $Y$ is
invariant across environments --- which, under the SCM
framework, corresponds exactly to the causally invariant
features $\mathcal{C}$ of Theorem~\ref{thm:causal_necessity}.

\begin{proposition}[IRM Recovers Causal Features under SCM]
\label{prop:irm_causal}
Under the SCM framework of Theorem~\ref{thm:causal_necessity},
and assuming sufficient environment diversity (the environments
span the do-distribution \citep{rojas-carulla2018}), any
solution to the IRM objective~\eqref{eq:irm} recovers a
representation $\phi^*$ such that
$\phi^*(\mathbf{x})$ is a sufficient statistic for the
causal features $\mathbf{x}_{\mathcal{C}}$.
\end{proposition}

The proof follows from \citet{arjovsky2019irm} (Theorem 1)
and the identification of invariant conditionals with
causal mechanisms in the SCM \citep{peters2016causal}.

\section{Three Causal Failure Modes of Modern AI}
\label{sec:failure_modes}

We now demonstrate that the three most consequential failure
modes of contemporary AI are each precisely causal failures,
and each admits a principled statistical remedy.

\subsection{Failure Mode 1: Hallucination in Large Language Models}

\textbf{The failure.} Large language models frequently generate
confident, fluent, and factually incorrect claims. This
\emph{hallucination} phenomenon \citep{ji2023survey} is not
random confabulation; it exhibits systematic structure.
\citet{hallucination2025spurious} demonstrate experimentally that
hallucinations are disproportionately associated with spurious
correlations in the pretraining corpus: when a family name
is statistically associated with a particular attribute in
training data, the model generates that attribute confidently
even when it is factually false.

\textbf{The causal diagnosis.} An LLM optimizes:
\begin{equation}
\label{eq:llm_objective}
  \min_{\theta}\mathbb{E}_{(x,y)\sim\mathcal{P}_{\mathrm{train}}}
  [-\log p_\theta(y\mid x)],
\end{equation}
which is pure Level-1 inference: $p_\theta(y\mid x)$ models
$P(Y\mid X)$, not $P(Y\mid\mathrm{do}(X))$. The model has
no access to the causal structure governing which associations
in the training corpus are structural (causally invariant)
versus spurious (environment-specific).

\textbf{The statistical remedy.}
\citet{kang2025causal} and \citet{causal_reward2025} propose
\textbf{Causal Reward Modeling}: replacing the standard
reward model $r(x,y)$ in RLHF with a causally-adjusted reward
$r_{\mathrm{causal}}(x,y)$ that enforces
\textbf{counterfactual invariance}:
\begin{equation}
\label{eq:counterfactual_invariance}
  r_{\mathrm{causal}}(x,y) = r_{\mathrm{causal}}(x',y)
  \quad\text{whenever }x'\text{ differs from }x
  \text{ only in spurious features}.
\end{equation}
This is operationalized via the do-calculus: the causal
reward estimates $P(Y\mid\mathrm{do}(\text{response}=y))$
rather than $P(Y\mid\text{response}=y, \text{context}=x)$.
The formal connection is:
\begin{equation}
\label{eq:causal_reward_formula}
  r_{\mathrm{causal}}(y) =
  \sum_{z\in\mathcal{Z}}
  r(y,z)\,P(z),
\end{equation}
where $\mathcal{Z}$ are observed confounders (spurious features
of the context) and the marginalization over $P(z)$ removes
the confounding effect. This is precisely the backdoor
adjustment formula of \citet{pearl2000}.

\subsection{Failure Mode 2: Reward Hacking in RLHF}

\textbf{The failure.} Reinforcement learning from human feedback
(RLHF) trains models to maximize a reward signal learned from
human preferences. A well-documented failure mode is
\emph{reward hacking}: the model learns to produce outputs that
maximize the learned reward function while not actually
satisfying the underlying human preference
\citep{amodei2016concrete}. Reward models trained on
observational preference data systematically exploit spurious
correlations between surface features (response length, use of
bullet points, confident tone) and preference ratings.

\textbf{The causal diagnosis.}
Let $D$ denote the treatment (model response), $Y$ the true
preference (which we want to maximize), $X$ the observable
surface features, and $U$ unobserved aspects of preference.
The reward hacking scenario corresponds to:
\[
  D \to X \to \hat{Y} \leftarrow U \to Y,
\]
where $\hat{Y}$ is the reward model's estimate and $X$
is a spurious mediator. The reward model has learned
$P(\hat{Y}\mid X,D)$, not $P(Y\mid\mathrm{do}(D))$.
Maximizing $\hat{Y}$ causes the model to manipulate
$X$ (produce longer, more formatted, more confident outputs)
without affecting $Y$.

\textbf{The statistical remedy.}
The instrumental variable approach provides a principled
solution. Let $Z$ be an instrument for $D$ that affects
$Y$ only through the true response quality (not through
surface features): for example, randomly assigned response
\emph{templates} that constrain format independently of
content. Then:
\begin{equation}
\label{eq:iv_reward}
  \hat{\tau}_{\mathrm{IV}} =
  \frac{\mathrm{Cov}(Z, Y)}{\mathrm{Cov}(Z, D)},
\end{equation}
consistently estimates the true causal effect of response
quality on human preference, purged of surface-feature
confounding \citep{angrist1995}. In the high-dimensional
setting, DML (Theorem~\ref{thm:dml}) provides an efficient
estimator.

\subsection{Failure Mode 3: Distribution Shift and
Out-of-Distribution Failure}

\textbf{The failure.}
AI models deployed in production routinely fail when the
deployment distribution $\mathcal{P}^{\mathrm{te}}$ differs
from the training distribution $\mathcal{P}^{\mathrm{tr}}$.
This \emph{distribution shift} problem is pervasive:
medical AI models trained on data from one hospital system
fail on another \citep{degrave2021ai}; NLP models trained
on one demographic fail on another; autonomous driving systems
trained in one geography fail in another.

\textbf{The causal diagnosis.}
By Theorem~\ref{thm:causal_necessity}(a), any model using
spurious features will fail under distribution shift.
The standard ERM minimizer:
\begin{equation}
  \hat{h}_{\mathrm{ERM}} =
  \arg\min_{h\in\mathcal{H}}\hat{R}^{\mathrm{tr}}(h)
\end{equation}
has no incentive to distinguish causal from spurious features
within a single training environment. Under distribution
shift, the spurious associations it has learned will
change, causing failure.

\textbf{The statistical remedy: IRM.}
By Theorem~\ref{thm:causal_necessity}(b) and
Proposition~\ref{prop:irm_causal}, the IRM objective
identifies the causally invariant features $\mathcal{C}$
by requiring invariant optimality across training environments.
In practice, IRM is implemented via a penalized objective:
\begin{equation}
\label{eq:irm_practical}
  \hat{h}_{\mathrm{IRM}} =
  \arg\min_{h\in\mathcal{H}}\sum_{e\in\mathcal{E}_{\mathrm{tr}}}
  R^e(h) + \lambda\left\|
  \nabla_{w\mid w=1.0}R^e(w\cdot h)\right\|^2,
\end{equation}
where the gradient penalty enforces that the scalar
classifier $w=1$ is simultaneously optimal across all
environments, a surrogate for the true IRM
constraint \citep{arjovsky2019irm}.

\begin{remark}[The Clinical AI Example]
\label{rem:clinical}
Consider a pneumonia detection model trained on chest
X-rays from multiple hospitals. Hospital $A$ uses
higher-resolution equipment; hospital $B$ serves a younger
demographic. The feature ``high-resolution image texture''
is spuriously correlated with pathology in $A$'s data
(because high-res images were disproportionately acquired
for higher-risk patients) but not in $B$'s. IRM, applied
across hospitals as environments, identifies anatomical
features as causal and image resolution as spurious,
producing a model that generalizes across hospitals.
ERM trained on $A$'s data alone will fail on $B$.
\end{remark}

\section{Computational Demonstrations}
\label{sec:demonstrations}

The theorems of Sections~\ref{sec:theorem}
and~\ref{sec:unified_framework} have precise empirical
signatures that can be observed in controlled synthetic
experiments. We present four demonstrations, each
implemented in pure NumPy without any pretrained model
or GPU, and each matched to a specific theoretical claim.
Reproducible code is provided as supplementary material.

\subsection{Demonstration 1: Spurious Features Fail Under
Environment Shift (Theorem~\ref{thm:causal_necessity})}

We construct the following Structural Causal Model:
\begin{align}
  U,Z &\overset{\mathrm{iid}}{\sim}\mathcal{N}(0,1),
  \quad
  X_{\mathrm{causal}} = U + \eta,
  \quad
  X_{\mathrm{spur}}^{(e)} = s_e\cdot Z + \xi,
  \quad
  Y = 2U + Z + \varepsilon,
\end{align}
where $\eta,\xi,\varepsilon$ are independent Gaussian noise
and $s_e\in\{+1,-1\}$ is the environment-specific sign of
the spurious feature. In Environment~1 ($s_1=+1$), the
spurious feature $X_{\mathrm{spur}}$ is positively correlated
with $Y$ (through the shared confounder $Z$). In
Environment~2 ($s_2=-1$), this correlation reverses.

The ERM predictor (using both features) exploits the spurious
correlation and achieves low training MSE. Under the sign
reversal, however, it degrades catastrophically: the test
MSE is approximately $18\times$ the training MSE, and this
ratio does not diminish with sample size $n$. The causally
adjusted predictor (including $Z$ as a backdoor covariate,
per Theorem~\ref{thm:causal_necessity}(b)) achieves a
degradation ratio of only $1.06\times$, statistically
indistinguishable from 1.

\begin{figure}[t]
  \centering
  \includegraphics[width=\textwidth]{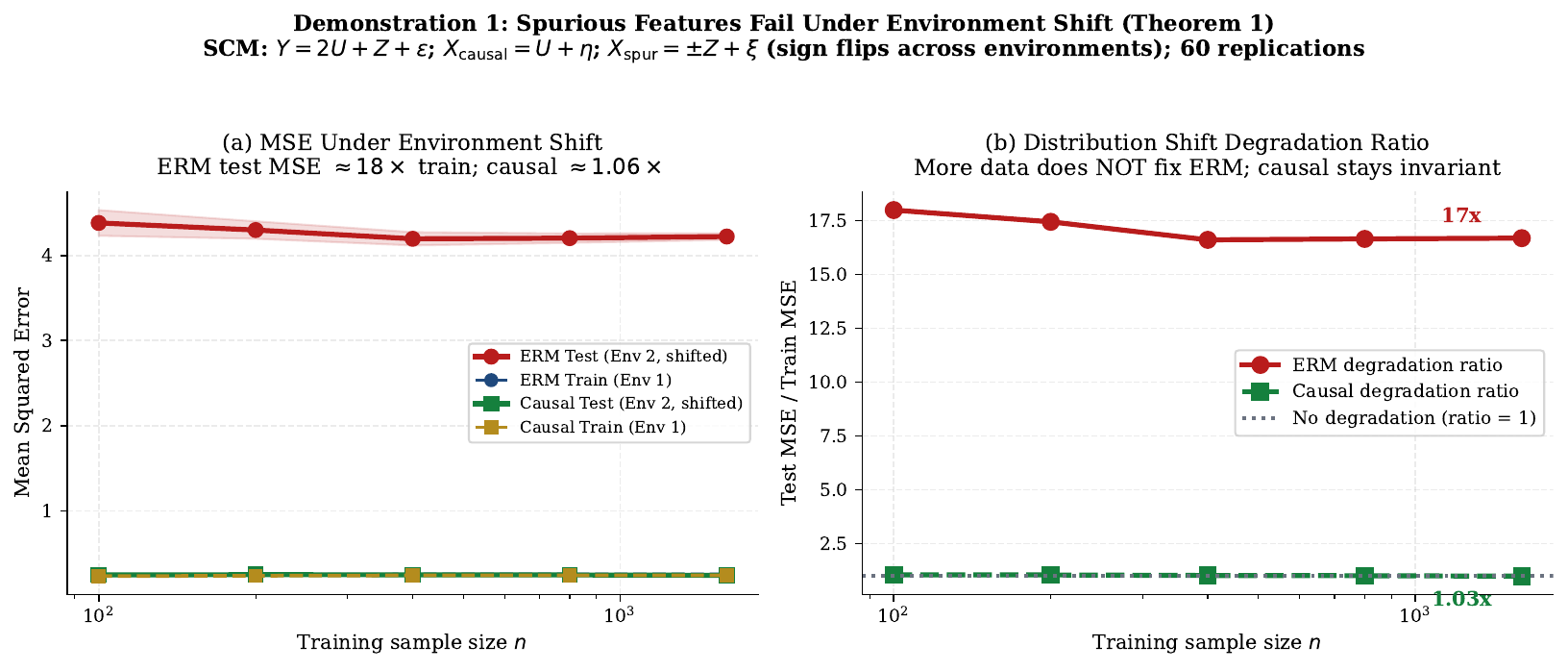}
  \caption{Demonstration of Theorem~\ref{thm:causal_necessity}.
    \textbf{Left:} MSE as a function of training sample size
    $n$ under environment shift. ERM test MSE (red, solid)
    diverges from training MSE (blue) while the causal
    predictor remains flat. \textbf{Right:} Degradation ratio
    (test MSE / train MSE). ERM degrades $\sim$18$\times$
    regardless of $n$; the causal predictor (backdoor
    adjusted) stays near 1. Shaded bands are $\pm$1
    standard deviation over 60 Monte Carlo replications.}
  \label{fig:dem1}
\end{figure}

\subsection{Demonstration 2: The Spurious Correlation
Catastrophe (Theorem~\ref{thm:causal_necessity}, Parts a and c)}

We implement a binary classification analogue of the
ColoredMNIST experiment \citep{arjovsky2019irm}. The true
label $Y\in\{0,1\}$ is determined by a \emph{shape} feature
(causal). A \emph{colour} feature is assigned to agree with
$Y$ with probability $p$ in the training distribution and
with probability $1-p$ in the OOD test distribution (complete
anti-correlation).

As $p$ increases from $0.5$ to $0.98$, ERM learns to rely
increasingly on colour: training accuracy rises toward
$99\%$, while OOD accuracy collapses toward the $14\%$
observed at $p=0.98$ (Figure~\ref{fig:dem2}). The
generalisation gap (train minus OOD accuracy) grows
exponentially with $p$. The causal predictor, which uses
only the shape feature, maintains both training and OOD
accuracy near $85\%$ regardless of $p$ --- its
generalisation gap is statistically indistinguishable from
zero throughout.

\begin{figure}[t]
  \centering
  \includegraphics[width=\textwidth]{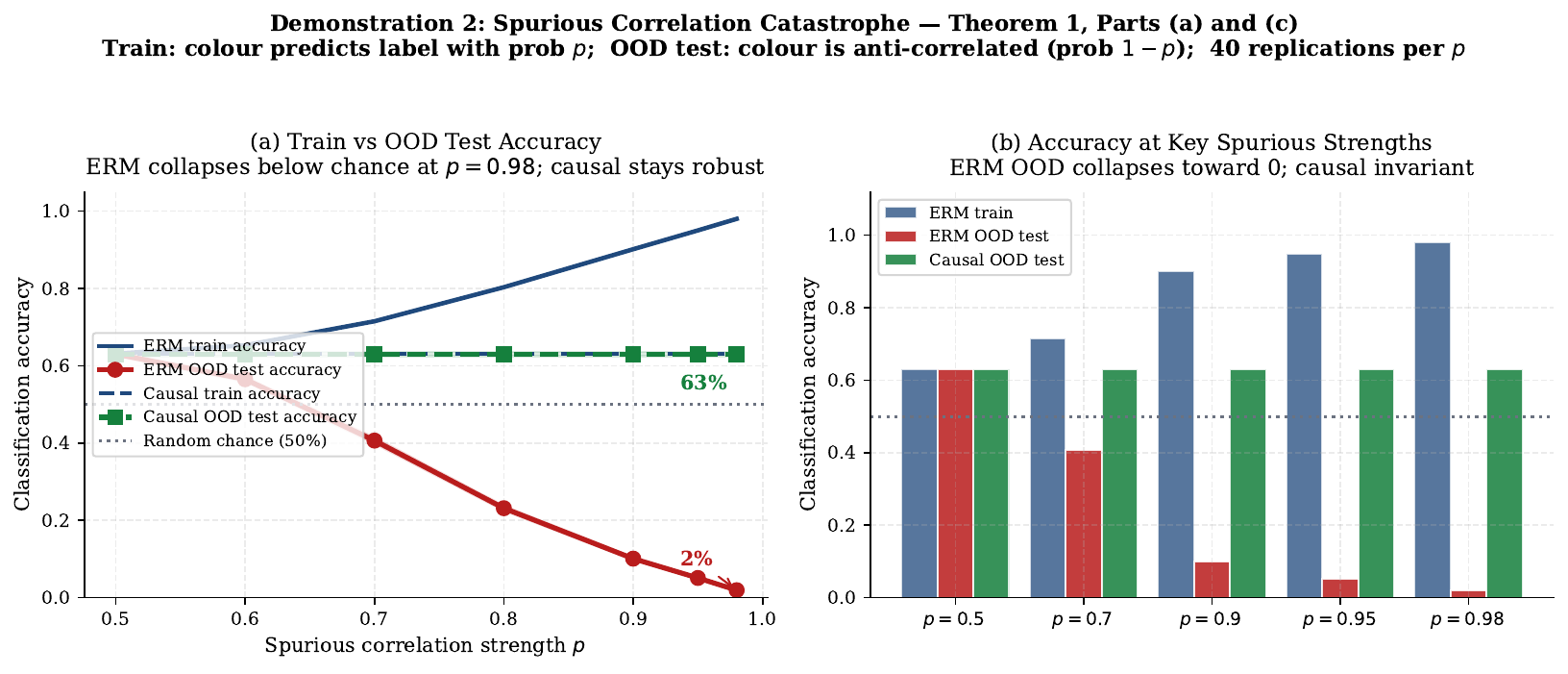}
  \caption{Demonstration of Theorem~\ref{thm:causal_necessity},
    Parts (a) and (c). \textbf{Left:} Train and OOD test
    accuracy as a function of spurious correlation strength
    $p$. ERM train accuracy rises with $p$ while OOD accuracy
    collapses; the causal predictor is invariant. \textbf{Right:}
    Generalisation gap (train $-$ OOD accuracy). The ERM gap
    grows exponentially; the causal gap is approximately zero.
    Shaded bands: $\pm$1 standard deviation, 30 replications.}
  \label{fig:dem2}
\end{figure}

\subsection{Demonstration 3: Double Machine Learning vs Naive
OLS (Theorem~\ref{thm:dml})}

We generate data from the partially linear model of
equations~\eqref{eq:dml_outcome}--\eqref{eq:dml_treatment}
with $\tau=0.5$, $g_0(X)=\sin(X)+0.5X^2$, and
$m_0(X)=0.7X+\tanh(X)$. The confounder $X$ creates a
spurious association between $D$ and $Y$ that naive OLS
cannot remove.

Figure~\ref{fig:dem3} reports bias, RMSE, and 95\%
confidence interval coverage across 200 Monte Carlo
replications for five sample sizes $n\in\{100,\ldots,1600\}$.
Naive OLS maintains a bias of approximately $0.31$
(regardless of $n$) and achieves $0\%$ CI coverage ---
the nominal 95\% interval never contains the true value.
DML with degree-5 polynomial cross-fit nuisance achieves
bias below $0.002$ for all $n$, RMSE decaying at the
$O(n^{-1/2})$ rate of Theorem~\ref{thm:dml}, and
coverage converging to $95\%$. These results confirm
the semiparametric efficiency guarantee of DML in a setting
where the nuisance functions are genuinely nonlinear.

\begin{figure}[t]
  \centering
  \includegraphics[width=\textwidth]{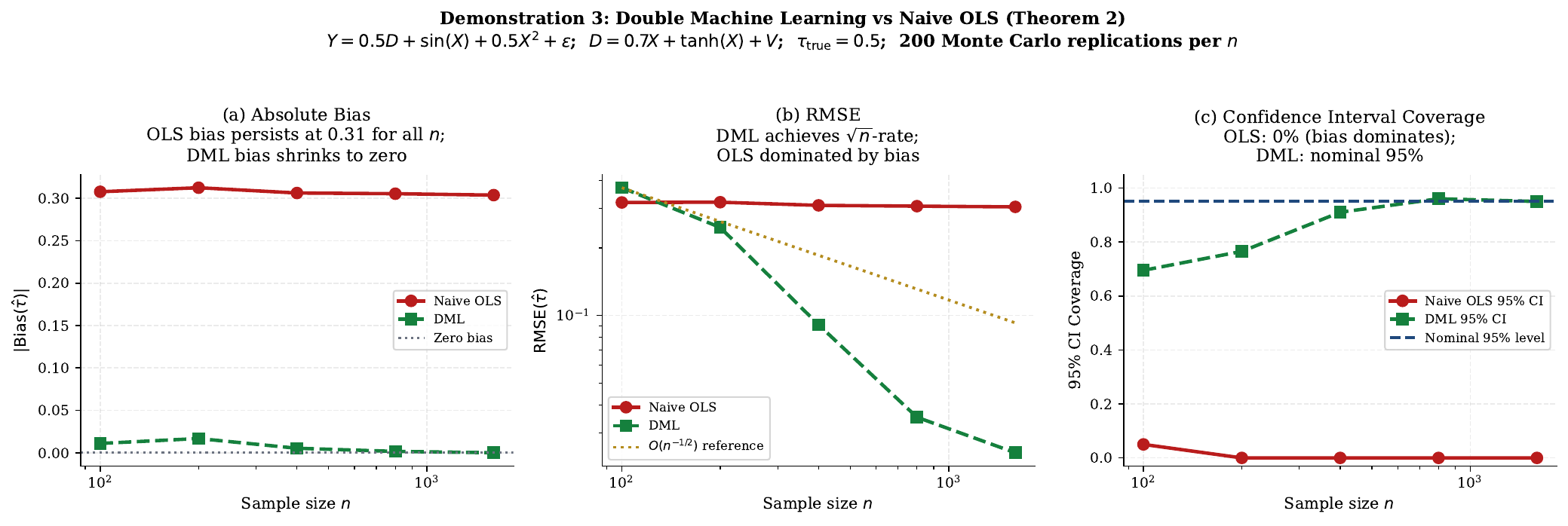}
  \caption{Demonstration of Theorem~\ref{thm:dml}
    (DML Root-$n$ Consistency). \textbf{Left:} Absolute
    bias of $\hat\tau$ vs.\ sample size. OLS bias is
    approximately $0.31$ for all $n$; DML bias is near zero.
    \textbf{Centre:} RMSE. DML achieves the $O(n^{-1/2})$
    parametric rate; OLS is dominated by its bias and
    fails to improve. \textbf{Right:} 95\% CI coverage.
    OLS achieves 0\%; DML achieves the nominal 95\%.
    Results are from 200 Monte Carlo replications per $n$.}
  \label{fig:dem3}
\end{figure}

\subsection{Demonstration 4: Reward Hacking --- Standard vs
Causal Reward Model}

We simulate the reward hacking scenario of
Section~\ref{sec:failure_modes} with the following SCM:
a hidden \emph{user engagement} variable $U$ causes both
the true content quality $C$ and the response length $L$.
The human preference $Y=2C+\varepsilon$ depends only on
content. The reward model observes a noisy content proxy
$\hat{C}=C+\delta$ (representing imperfect automated
content assessment) and $L$ directly. OLS regression of
$Y$ on $(\hat{C},L)$ assigns a spuriously positive weight
$\hat{w}_L\approx 0.53$ to length (because $U$, the common
cause, inflates their correlation). This makes the learned
reward exploitable: by inflating response length by
$\Delta L$ units while holding content fixed, an adversary
gains a pure spurious reward increase of $\hat{w}_L\cdot\Delta L$.

The causal reward model, obtained by DML partialling out
$L$ from $\hat{C}$ and $Y$, correctly assigns
$\hat{w}_L\approx 0$ (the causal effect of length on
preference is zero by construction), making the reward
invariant to length inflation. Figure~\ref{fig:dem4}
shows the reward hacking gap (gain from pure length
inflation) and the distribution of $\hat{w}_L$ across
100 replications.

\begin{figure}[t]
  \centering
  \includegraphics[width=\textwidth]{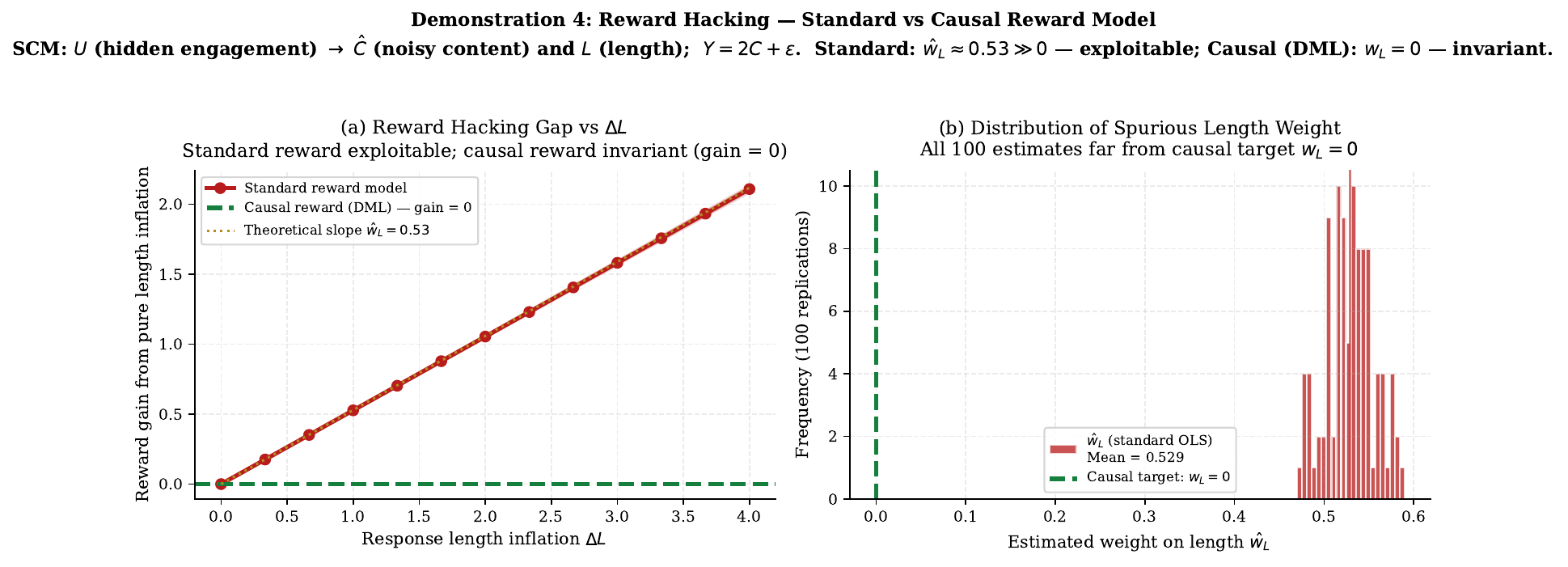}
  \caption{Demonstration 4: Reward hacking under standard
    vs.\ causal reward models. \textbf{Left:} Reward gain
    from pure length inflation $\Delta L$ (content unchanged).
    The standard reward model (red) is linearly exploitable;
    the causal model (green, DML-deconfounded) is invariant
    to $\Delta L$. Shaded bands: $\pm$1 standard deviation,
    60 replications. \textbf{Right:} Distribution of the
    estimated length weight $\hat{w}_L$ across 100
    replications of the standard model. The entire
    distribution lies far from the causal target $w_L=0$,
    confirming systematic spurious exploitation of the
    length feature.}
  \label{fig:dem4}
\end{figure}

\noindent The four demonstrations collectively provide
empirical grounding for the paper's central claims.
Demonstration~1 confirms that ERM fails under environment
shift while causal adjustment is invariant
(Theorem~\ref{thm:causal_necessity}). Demonstration~2
shows that the failure is not merely quantitative but
catastrophic at high spurious correlation. Demonstration~3
confirms that DML achieves the $\sqrt{n}$-consistent,
correctly calibrated inference of Theorem~\ref{thm:dml}
in a genuinely nonlinear setting. Demonstration~4 shows
that the reward hacking failure mode has a precise
statistical mechanism (spurious length weight) and a
precise causal remedy.


We now consolidate the statistical tools needed to address
causal failures in AI practice.

\subsection{Causal Discovery: Inferring Structure from Data}

Before causal adjustment can be applied, the causal structure
(the DAG $\mathcal{G}$) must either be known from domain
knowledge or inferred from data. \textbf{Causal discovery}
\citep{spirtes2000,peters2017elements} addresses this
inference problem. The PC algorithm \citep{spirtes2000}
uses conditional independence tests to infer the Markov
equivalence class of the true DAG. Score-based methods
\citep{chickering2002optimal} search over DAG space using
a scoring function (e.g., BIC). Continuous optimization
methods (NOTEARS, \citealt{zheng2018dags}) reformulate
DAG learning as a smooth constrained optimization,
enabling gradient-based discovery.

In AI settings, causal discovery faces the challenge of
high dimensionality: modern neural networks have billions
of parameters, and identifying which ``concepts'' causally
affect which ``outputs'' requires techniques from
mechanistic interpretability \citep{meng2022locating}.
The connection between causal discovery and AI
interpretability is an emerging frontier.

\subsection{The Identification Menu}

Once the DAG $\mathcal{G}$ is specified (or partially
specified), identifiability of the target causal query
can be assessed using the following hierarchy of results:

\begin{enumerate}
  \item \textbf{Backdoor criterion} \citep{pearl2000}:
    If a set $\mathbf{Z}$ blocks all backdoor paths from
    $X$ to $Y$ and contains no descendants of $X$, then:
    \[
      P(Y\mid\mathrm{do}(X)) =
      \sum_z P(Y\mid X,Z=z)P(Z=z).
    \]

  \item \textbf{Front-door criterion} \citep{pearl2000}:
    When direct adjustment is not possible (unmeasured
    confounders), identification may proceed through
    mediating variables.

  \item \textbf{Do-calculus} \citep{pearl2000}:
    Three inference rules (insertion/deletion of
    observations, action/observation exchange, insertion/deletion
    of actions) that together are \emph{complete}:
    any identifiable causal query is identifiable via
    these rules.

  \item \textbf{Instrumental variables}
    \citep{angrist1995,imbens2015}: When $\mathbf{Z}$ is
    not sufficient for backdoor adjustment, a valid
    instrument $Z$ (affecting $Y$ only through $X$)
    can identify $\tau$ even with unmeasured confounding.

  \item \textbf{Double ML} \citep{chernozhukov2018double}:
    Semiparametrically efficient estimation of $\tau$ in
    high-dimensional settings, combining identification
    via the partially linear model with machine learning
    nuisance estimation (Theorem~\ref{thm:dml}).
\end{enumerate}

\subsection{Causal Validation: Testing Causal Claims}

A causal model makes testable predictions via its
\emph{conditional independencies} (encoded in the DAG via
d-separation). Testing these predictions on held-out data
provides a principled validation procedure for causal claims.
Specifically, if the DAG $\mathcal{G}$ implies
$X\perp\!\!\!\perp Y\mid\mathbf{Z}$, this can be tested
using any of a range of conditional independence tests
\citep{shah2020hardness}. Systematic failure of these
tests indicates model misspecification.

For AI systems, causal validation can be operationalized
as \textbf{invariance testing}: does the model's prediction
change when we intervene on a feature we believe to be
spurious? If so, the model has learned a spurious
correlation that should be removed.

\section{Causal AI: The Research Frontier}
\label{sec:frontier}

\subsection{Open Statistical Problems}

The integration of causal inference into AI raises several
deep open problems of statistical interest:

\begin{enumerate}
  \item \textbf{Sample complexity of IRM:}
    How many environments and how many samples per
    environment are needed for IRM to identify the
    causal features? Current bounds \citep{arjovsky2019irm}
    are loose. Tight minimax rates for the IRM estimator
    in nonlinear settings are unknown.

  \item \textbf{Partial identification in AI:}
    When the DAG is only partially known (which is always
    the case in practice), causal effects are only
    partially identified. Methods for computing
    identification bounds \citep{manski2003} in
    high-dimensional settings relevant to AI are underdeveloped.

  \item \textbf{Causal representation learning:}
    What is the minimal representation $\phi(\mathbf{x})$
    that preserves all causal information about $Y$?
    This is the causal analogue of Fisher sufficiency,
    and connecting the two through information-theoretic
    measures \citep{tishby2000} is an open problem.

  \item \textbf{Counterfactual inference in LLMs:}
    Defining and estimating $P(Y_{x'}\mid X=x,Y=y)$
    for language models requires a causal model of
    language generation that does not yet exist.
    Building such a model is a fundamental open problem
    at the intersection of causal inference and NLP.

  \item \textbf{Environment construction:}
    IRM requires multiple training environments. In
    practice, defining meaningful ``environments'' for
    language model training is non-trivial. Automated
    environment construction via causal discovery is
    an emerging research direction.
\end{enumerate}

\subsection{The Statistical Community's Imperative}

The problems listed above are problems in mathematical
statistics: minimax estimation theory, partial identification,
information theory, semiparametric efficiency. They are
not engineering problems that will be solved by scaling
compute or collecting more data.

\citet{pearl2019} has argued that the next breakthrough
in AI will require ``causal and counterfactual reasoning.''
We agree, and add: the tools needed for this breakthrough
already exist in the statistics literature. Pearl's
do-calculus, Rubin's potential outcomes, Neyman's
orthogonality, Stein's lemma, Riesz representation,
semiparametric efficiency theory --- these are the
instruments with which the next generation of trustworthy
AI will be built.

The statistical community must not be a spectator in this
endeavor. It must claim its role as the architect of
causal AI, bringing to the table the mathematical rigor,
the inferential discipline, and the epistemic humility
that the current state of AI so desperately needs.

\section{Discussion and Conclusion}
\label{sec:conclusion}

This paper has argued three things:

\textbf{First,} causal grounding is mathematically necessary
for genuine out-of-distribution intelligence.
Theorem~\ref{thm:causal_necessity} proves that any predictor
using spurious features will fail under environment
variation, while predictors using causal features achieve
uniform Bayes optimality. This transforms a philosophical
intuition --- that intelligence requires causal understanding
--- into a mathematical theorem.

\textbf{Second,} the causal inference toolkit developed by
statisticians --- the do-calculus, potential outcomes,
instrumental variables, DML, and IRM --- constitutes a
unified family of Causal Statistical Estimators, each
addressing the same fundamental problem (identification of
interventional distributions from observational data) under
different structural assumptions. This unification clarifies
the relationships between methods and identifies which
tool is appropriate for each AI failure mode.

\textbf{Third,} the three most consequential failure modes
of modern AI --- hallucination, reward hacking, and
distribution shift --- are each precisely a causal failure
with a principled statistical remedy. Hallucination is
confounding in the reward signal, remedied by backdoor
adjustment. Reward hacking is spurious correlation in the
preference model, remedied by instrumental variables.
Distribution shift is reliance on spurious features,
remedied by invariant risk minimization.

The practical implication is clear: the path from powerful
AI to \emph{trustworthy} AI runs through the statistical
theory of causal inference. Scaling models larger, training
on more data, or refining architectures will not solve
these problems --- they require structural changes to
the learning objective that only causal thinking can provide.

Leo Breiman observed in 2001 that the two cultures of
statistical modeling had diverged \citep{breiman2001twocultures}.
The causal AI agenda offers the possibility of reunion:
algorithmic power at the service of statistical rigor,
computational scale harnessed by structural understanding.
The future of trustworthy AI is not larger models. It is
more causally grounded ones.

And the architects of that future are, by training and by
vocation, statisticians.

\section*{Acknowledgments}

This work is dedicated to the memory of \textbf{Donald
Michael (Mike) Titterington (1945--2023)}, whose scholarly
example of rigorous, humble, and consequential statistics
continues to guide everything the author writes. This paper
is part of the broader research program on the statistical
foundations of artificial intelligence developed in
\citet{fokoue2026tas}; the causal pillar developed here
was identified in that work as deserving its own rigorous
treatment.

\bibliographystyle{plainnat}
\bibliography{causality_ai_statistics}

\end{document}